\documentclass[11pt,a4paper]{article}

\usepackage[utf8]{inputenc}
\usepackage[T1]{fontenc}
\usepackage[spanish,english]{babel}
\usepackage[a4paper,margin=2.5cm]{geometry}
\usepackage{amsmath,amssymb,amsthm}
\usepackage{graphicx}
\usepackage{xcolor}
\usepackage{booktabs}
\usepackage{microtype}
\usepackage{setspace}
\usepackage{titlesec}
\usepackage{fancyhdr}
\usepackage{hyperref}
\usepackage{lmodern}

\newtheorem{theorem}{Theorem}[section]
\newtheorem{Lemma}[theorem]{Lemma}
\newtheorem{Proposition}[theorem]{Proposition}

\theoremstyle{remark}

\newtheorem{Example}[theorem]{Example}

\hypersetup{
    colorlinks=true,
    linkcolor=blue,
    citecolor=blue,
    urlcolor=blue
}

\titleformat{\section}{\large\bfseries}{\thesection.}{0.5em}{}
\titleformat{\subsection}{\normalsize\bfseries}{\thesubsection.}{0.5em}{}

\begin{document}

\begin{titlepage}
    \centering
    \vspace{1.5cm}
    {\huge\bfseries Step-Size Decay and Structural Stagnation in Greedy Sparse Learning\par}
    \vspace{1cm}
    {\Large Pablo M. Berná\par}
    \vspace{0.3cm}
    {\normalsize Departamento de Matemáticas, CUNEF Universidad\par}
    {\normalsize Madrid, Spain\par}
    {\normalsize \texttt{pablo.berna@cunef.edu}\par}
    \vfill
    {\large \today\par}
\end{titlepage}

\begin{abstract}
Greedy algorithms are central to sparse approximation and stage-wise
learning methods such as matching pursuit and
boosting. It is known that
the Power-Relaxed Greedy Algorithm with step sizes $m^{-\alpha}$ may fail
to converge when $\alpha>1$ in general Hilbert spaces. In this work, we
revisit this phenomenon from a sparse learning perspective. We study
realizable regression problems with controlled feature coherence and
derive explicit lower bounds on the residual norm, showing that
over-decaying step-size schedules induce structural stagnation even in
low-dimensional sparse settings. Numerical experiments confirm the
theoretical predictions and illustrate the role of feature coherence.
Our results provide insight into step-size design in greedy sparse learning.
\end{abstract}

\tableofcontents
\vspace{1em}

\section{Introduction}

Greedy algorithms play a central role in approximation theory and, more recently, in several areas of machine learning. Their basic principle is simple: starting from a dictionary of atoms, one iteratively builds an approximation by selecting, at each step, the element that is most correlated with the current residual. This idea goes back at least to Matching Pursuit \cite{ref-mallat} and has since become a standard tool in sparse approximation \cite{ref-tropp,ref-temlyakov-book}. Variants of this scheme also appear in orthogonal matching pursuit \cite{ref-omp}, forward stage-wise regression and boosting methods \cite{ref-friedman,ref-lars}, as well as in projection-free optimization algorithms of the Frank--Wolfe type \cite{ref-fw,ref-jaggi}.

In the Hilbert space setting, a classical formulation is the Relaxed Greedy Algorithm (RGA) introduced by DeVore and Temlyakov \cite{ref-devore}. The RGA combines greedy atom selection with a convex relaxation step of size $1/m$ at iteration $m$. Under suitable atomic norm assumptions, this schedule guarantees convergence and yields optimal approximation rates. From a structural point of view, the factor $1/m$ plays a crucial role: it balances the aggressiveness of the greedy choice with a sufficiently large cumulative correction.

A natural modification replaces the step size $1/m$ by $1/m^\alpha$, leading to the so-called Power-Relaxed Greedy Algorithm (PRGA). It is known that convergence persists when $\alpha \le 1$, while for $\alpha>1$ non-convergence may occur; see \cite{ref-garcia,ref-prga}. From a purely functional-analytic perspective, this settles the general question. However, the implications of this phenomenon in sparse learning contexts have not been explored in detail.

Let us briefly recall the definitions. Let $H$ be a real Hilbert space with inner product $\langle\cdot,\cdot\rangle$ and norm $\|\cdot\|$, and let $\mathcal D \subset H$ be a symmetric dictionary of unit-norm elements. Given $f \in H$, we construct approximations $(f_m)_{m\ge 0}$ with residuals $r_m = f - f_m$.

\medskip
\noindent
\textbf{Relaxed Greedy Algorithm (RGA, \cite{ref-devore}).}
Let $H$ be a Hilbert space and $\mathcal D\subset H$ a symmetric dictionary of unit vectors.
Given $f\in H$, set $f_0=0$ and $r_0=f$.

\smallskip
\noindent
\emph{Step $m=1$.} Choose
\[
g_1\in\mathcal D \quad \text{such that} \quad
\langle r_0,g_1\rangle=\sup_{g\in\mathcal D}\langle r_0,g\rangle,
\]
and define
\[
f_1 := \langle r_0,g_1\rangle g_1,
\qquad
r_1 := f-f_1 .
\]

\smallskip
\noindent
\emph{Step $m\ge2$.} Choose
\[
g_m\in\mathcal D \quad \text{such that} \quad
\langle r_{m-1},g_m\rangle=\sup_{g\in\mathcal D}\langle r_{m-1},g\rangle,
\]
and update
\[
f_m = \left(1-\frac{1}{m}\right) f_{m-1} + \frac{1}{m}\, g_m,
\qquad
r_m = f - f_m .
\]

\medskip
\noindent
\textbf{Power-Relaxed Greedy Algorithm (PRGA, \cite{ref-garcia}).}
Let $\alpha>0$ and define $\lambda_m = m^{-\alpha}$.
We again set $f_0=0$ and $r_0=f$.

\smallskip
\noindent
\emph{Step $m=1$.} Choose
\[
g_1\in\mathcal D \quad \text{such that} \quad
\langle r_0,g_1\rangle=\sup_{g\in\mathcal D}\langle r_0,g\rangle,
\]
and define
\[
f_1 := \langle r_0,g_1\rangle g_1,
\qquad
r_1 := f-f_1 .
\]

\smallskip
\noindent
\emph{Step $m\ge2$.} Choose
\[
g_m\in\mathcal D \quad \text{such that} \quad
\langle r_{m-1},g_m\rangle=\sup_{g\in\mathcal D}\langle r_{m-1},g\rangle,
\]
and update
\[
f_m = (1-\lambda_m) f_{m-1} + \lambda_m g_m,
\qquad
r_m = f - f_m .
\]

When $\alpha=1$, PRGA reduces to the classical RGA.

\medskip

From the viewpoint of machine learning, the choice of step-size schedule is often interpreted through the lens of stability: in gradient-based methods, decreasing step sizes are typically associated with improved convergence properties, and classical stochastic approximation theory requires $\sum_m \lambda_m = \infty$ for convergence \cite{ref-robbins}. Greedy methods, however, differ in nature from gradient descent: they rely on directional selection rather than smooth descent along the gradient. It is therefore not a priori clear how aggressive step-size decay affects their long-term behavior.

Greedy sparse adaptive algorithms have also been extensively studied in signal processing and online learning contexts, where adaptive step-size strategies and sparsity-promoting updates play a central role; see, for example, \cite{RefIEEE6151185}. In such settings, step-size decay is typically motivated by stability and tracking considerations.

This leads to a concrete question: can a rapidly decaying schedule, such as $m^{-\alpha}$ with $\alpha>1$, prevent convergence even in simple sparse regression problems?

In this work we address this question from a sparse learning perspective. We consider a realizable regression setting with two atoms of prescribed coherence and analyze the behavior of PRGA when $\alpha>1$. Our main contribution is a quantitative stagnation result showing that the residual norm remains bounded away from zero. We derive an explicit lower bound depending on the feature coherence and on the infinite product
\[
P_\alpha = \prod_{k=2}^\infty \left(1-\frac{1}{k^\alpha}\right),
\]
which is strictly positive for $\alpha>1$.

This demonstrates that overly fast step-size decay can induce a structural stagnation phenomenon even in low-dimensional, noiseless, perfectly realizable sparse regression models. The phenomenon is purely algorithmic: it does not stem from statistical complexity, lack of model expressivity, or adverse data distributions, but from the fact that when $\sum_m m^{-\alpha}<\infty$ the cumulative corrective mass introduced by the relaxation step is insufficient to eliminate the residual.

In addition to the theoretical analysis, we provide numerical experiments illustrating the dependence of the stagnation level on the coherence parameter and validating the predicted dependence on the product $P_\alpha$.

Taken together, these results clarify the role of step-size design in greedy sparse approximation and highlight a structural distinction between greedy schemes and classical gradient-based methods.

\medskip

The remainder of the paper is organized as follows. 
Section~2 presents the main stagnation result for PRGA in a two-atom sparse regression setting. 
Section~3 analyzes the dependence of the stagnation level on the coherence parameter and derives the explicit lower bound involving the product $P_\alpha$. 
Section~4 presents numerical experiments validating the theoretical predictions.

\section{Main theoretical result}

In this section we establish a quantitative stagnation result for the
Power--Relaxed Greedy Algorithm (PRGA) in a simple sparse regression
setting involving two atoms. The result shows that when the step-size
schedule decays too rapidly ($\alpha>1$), the algorithm cannot fully
eliminate the residual even in a realizable low-dimensional problem.

\begin{theorem}\label{main}
Consider the Euclidean space $(\mathbb R^n, \|\cdot\|_2)$. 
Let $\alpha>1$ and define $\lambda_m = m^{-\alpha}$. 
Let $x_1,x_2\in\mathbb R^n$ be unit vectors,
$\|x_1\|_2=\|x_2\|_2=1$, with coherence
\[
\mu := |\langle x_1,x_2\rangle| \in [0,1).
\]

Consider the symmetric dictionary
\[
\mathcal D = \{\pm x_1,\pm x_2\},
\]
and the realizable target
\[
y = (1-b)x_1 + b x_2,
\qquad b\in(0,1/2).
\]

Run the Power--Relaxed Greedy Algorithm (PRGA) over $\mathcal D$
with initialization $f_0=0$ and residual $r_0=y$.

Then the residual cannot converge to zero. In fact,
\[
\inf_{m\ge1}\|r_m\|_2
\;\ge\;
b(1-\mu)\sqrt{\frac{1+\mu}{2}}\,P_\alpha
\;>\;0,
\]
where
\[
P_\alpha := \prod_{k=2}^{\infty}\left(1-\frac{1}{k^\alpha}\right)\in(0,1).
\]

Moreover, if $\mu\in[0,1/2]$,
\[
\inf_{m\ge1}\|r_m\|_2
\;\ge\;
\frac{bP_\alpha}{\sqrt2}
-
\frac{3bP_\alpha}{4\sqrt2}\,\mu.
\]
\end{theorem}

\begin{proof}

Since the dictionary is symmetric, if $\langle x_1,x_2\rangle<0$ we may replace $x_2$ by $-x_2$
(which belongs to $\mathcal D$) and simultaneously rewrite $y=(1-b)x_1+b(-x_2)$ with the same $b$.
Therefore, without loss of generality we assume
\[
\langle x_1,x_2\rangle=\mu\in[0,1).
\]
In order to quantify the residual component that cannot be eliminated by convex combinations of dictionary atoms under finite accumulation, it is natural to measure size relative to the convex hull of the dictionary. Let $
B := \mathrm{conv}(\mathcal D).$ Since $\mathcal D$ is symmetric, $B$ is a symmetric convex body in $\mathrm{span}\{x_1,x_2\}$. We define its Minkowski functional (atomic norm)\footnote{From a machine learning perspective, the atomic norm associated with $\mathcal D$
measures the minimal “mass” of dictionary atoms required to represent a vector.
In particular, when the cumulative step size $\sum_m \lambda_m$ is finite,
the iterates $f_m$ remain confined to a scaled copy of the convex hull
$\mathrm{conv}(\mathcal D)$.
Thus, $\|\cdot\|_{\mathcal A}$ naturally quantifies the effective region reachable
by the algorithm under finite accumulation.} by
\[
\|u\|_{\mathcal A} := \inf\{t>0 : u \in tB\}.
\]

Then $B = \{u : \|u\|_{\mathcal A} \le 1\},$ and $\|\cdot\|_{\mathcal A}$ defines a norm on $\mathrm{span}\{x_1,x_2\}$. Moreover, by construction,
\[
\|g\|_{\mathcal A} = 1
\quad \text{for all } g \in \mathcal D.
\]
Define $s_m:=\|f_m\|_{\mathcal A}$ and $u_m:=1-s_m$.
For $m\ge2$, using the PRGA update
$f_m=(1-\lambda_m)f_{m-1}+\lambda_m g_m$ and $\|g_m\|_{\mathcal A}=1$, we have
\[
s_m
=\|(1-\lambda_m)f_{m-1}+\lambda_m g_m\|_{\mathcal A}
\le (1-\lambda_m)\|f_{m-1}\|_{\mathcal A}+\lambda_m\|g_m\|_{\mathcal A}
=(1-\lambda_m)s_{m-1}+\lambda_m .
\]
Equivalently, for $m\ge2$,
\[
u_m=1-s_m \ge (1-\lambda_m)(1-s_{m-1})=(1-\lambda_m)u_{m-1}.
\]
Since for $m\ge2$ we have $u_m \ge (1-\lambda_m)u_{m-1}$, we can write the first steps explicitly.
For $m=2$,
\[
u_2 \ge (1-\lambda_2)u_1.
\]
For $m=3$,
\[
u_3 \ge (1-\lambda_3)u_2
\ge (1-\lambda_3)(1-\lambda_2)u_1.
\]
Proceeding inductively, for every $m\ge2$,
\begin{equation}\label{eq:deficit-product}
u_m \ge u_1\prod_{k=2}^{m}(1-\lambda_k)
= u_1\prod_{k=2}^{m}\left(1-\frac{1}{k^\alpha}\right).
\end{equation}

In this part of the proof, we first show the first greedy choice is $g_1=x_1$. Indeed,
\[
\langle y,x_1\rangle=(1-b)+b\mu,\qquad
\langle y,x_2\rangle=b+(1-b)\mu,
\]
and thus
\[
\langle y,x_1\rangle-\langle y,x_2\rangle=(1-2b)(1-\mu)>0
\quad\text{since }b<1/2\text{ and }\mu<1.
\]
Therefore $g_1=x_1$ and $f_1= x_1\langle y,x_1\rangle
=\langle y,x_1\rangle\,x_1=((1-b)+b\mu)\,x_1.$

Because $\|t x_1\|_{\mathcal A}=|t|$ (as $x_1$ is an atom), we obtain
\[
s_1=\|f_1\|_{\mathcal A}= (1-b)+b\mu
\quad\Longrightarrow\quad
u_1=1-s_1=b(1-\mu).
\]
Plugging into \eqref{eq:deficit-product} gives
\begin{equation}\label{eq:um-lb}
1-\|f_m\|_{\mathcal A} = u_m \ge b(1-\mu)\prod_{k=2}^{m}\left(1-\frac{1}{k^\alpha}\right).
\end{equation}

Let $\|\cdot\|_{\mathcal A}^\ast$ denote the dual norm on $\mathrm{span}\{x_1,x_2\}$.
For a symmetric atomic norm generated by $\mathcal D$, one has the explicit dual characterization
\[
\|v\|_{\mathcal A}^\ast=\max_{g\in\mathcal D}\langle v,g\rangle=\max\{|\langle v,x_1\rangle|,|\langle v,x_2\rangle|\}.
\]
Define the unit vector
\[
v := \frac{x_1+x_2}{\|x_1+x_2\|_2},\qquad \|x_1+x_2\|_2=\sqrt{2(1+\mu)}.
\]
The last equality is derivated from here:
\[
\|x_1+x_2\|_2^2=\|x_1\|_2^2+\|x_2\|_2^2+2\langle x_1,x_2\rangle=2+2\mu.
\]
Then
\[
\langle v,x_1\rangle=\langle v,x_2\rangle=\sqrt{\frac{1+\mu}{2}},
\]
hence
\[
\|v\|_{\mathcal A}^\ast=\sqrt{\frac{1+\mu}{2}}.
\]
By duality, for any $m$,
\[
\langle f_m,v\rangle \le \|f_m\|_{\mathcal A}\,\|v\|_{\mathcal A}^\ast
=\|f_m\|_{\mathcal A}\sqrt{\frac{1+\mu}{2}}.
\]
On the other hand,
\[
\langle y,v\rangle
=(1-b)\langle x_1,v\rangle + b\langle x_2,v\rangle
=\sqrt{\frac{1+\mu}{2}}.
\]
Therefore,
\[
\langle r_m,v\rangle=\langle y-f_m,v\rangle
\ge \sqrt{\frac{1+\mu}{2}}\Bigl(1-\|f_m\|_{\mathcal A}\Bigr).
\]
Since $\|v\|_2=1$, Cauchy--Schwarz gives $\|r_m\|_2\ge \langle r_m,v\rangle$, hence using
\eqref{eq:um-lb} we obtain, for all $m\ge1$,
\[
\|r_m\|_2 \ge
\sqrt{\frac{1+\mu}{2}}\;b(1-\mu)\prod_{k=2}^{m}\left(1-\frac{1}{k^\alpha}\right).
\]

Finally, for $\alpha>1$, since $\prod_{k=2}^{m}\left(1-\frac{1}{k^\alpha}\right)$ is bigger than the infinite product
$P_\alpha=\prod_{k=2}^\infty(1-1/k^\alpha)$ that converges to a number in $(0,1)$
(see Lemma 3.1 in \cite{ref-prga}), taking $\inf_{m\ge1}$ we obtain
\[
\inf_{m\ge1}\|r_m\|_2 \ge b(1-\mu)\sqrt{\frac{1+\mu}{2}}\,P_\alpha>0.
\]

\noindent\textbf{Case for $\mu\in [0,1/2]$.} As we know that for $t\in[0,1]$ we have $\sqrt{1+t}\ge 1+\frac{t}{2}-\frac{t^2}{2}$, for $\mu\in[0,1/2]$ we have
\[
\sqrt{1+\mu}\ge 1+\frac{\mu}{2}-\frac{\mu^2}{2}.
\]
Since $\mu\in[0,1/2]$, we have $\mu^2\le \mu/2$, hence
\[
-\frac{\mu^2}{2}\ge -\frac{1}{2}\cdot\frac{\mu}{2}=-\frac{\mu}{4},
\]
and therefore
\[
\sqrt{1+\mu}\ge 1+\frac{\mu}{2}-\frac{\mu}{4}=1+\frac{\mu}{4}.
\]
Multiplying by $(1-\mu)\ge 0$ yields
\[
(1-\mu)\sqrt{1+\mu}\ge (1-\mu)\left(1+\frac{\mu}{4}\right).
\]
Expanding the product gives
\[
(1-\mu)\left(1+\frac{\mu}{4}\right)
=1+\frac{\mu}{4}-\mu-\frac{\mu^2}{4}
=1-\frac{3\mu}{4}-\frac{\mu^2}{4}.
\]
Finally, since $\mu^2\ge 0$, we conclude
\[
1-\frac{3\mu}{4}-\frac{\mu^2}{4}\ge 1-\frac{3\mu}{4},
\]
and hence
\[
(1-\mu)\sqrt{1+\mu}\ge 1-\frac{3\mu}{4},
\]
and multiplying by $b/\sqrt2$ yields the stated linear bound in $\mu$.
\end{proof}

Our theorem shows that over-decaying step-size schedules ($\alpha>1$) introduce a structural bias in greedy sparse learning methods, preventing full recovery even in realizable low-dimensional problems with small feature coherence. This highlights a fundamental distinction between greedy approximation schemes and classical gradient methods, and provides theoretical guidance for step-size design in stage-wise learning algorithms.

The stagnation phenomenon observed in Theorem~\ref{main} can be
understood in a broader geometric framework.
The following result relates the Euclidean norm and the atomic norm
for vectors supported on a small set of incoherent atoms.
This provides a mechanism through which lower bounds in atomic norm
translate into Euclidean residual floors in sparse settings.

\begin{Proposition}\label{prop:s-sparse-floor}
Let $H=\mathbb R^n$ with the Euclidean norm $\|\cdot\|_2$, and let
$S=\{x_{i_1},\dots,x_{i_s}\}\subset\mathcal D$ be a set of $s$ unit atoms
with pairwise coherence
\[
\mu_S:=\max_{p\neq q}|\langle x_{i_p},x_{i_q}\rangle|
<\frac{1}{s-1}.
\]

Assume that $y\in\mathrm{span}(S)$ and consider iterates
$f_m\in\mathrm{span}(S)$.
Let $r_m:=y-f_m$ and let $\|\cdot\|_{\mathcal A}$ be the atomic norm
induced by the symmetric dictionary $\mathcal D$.

Then, for every $m\ge0$,
\[
\|r_m\|_2
\;\ge\;
\frac{\sqrt{1-(s-1)\mu_S}}{\sqrt{s}}\,
\Bigl(\|y\|_{\mathcal A}-\|f_m\|_{\mathcal A}\Bigr).
\]
\end{Proposition}

\begin{proof}
Since $g_m\in \mathrm{span}(S)$ by assumption and $f_0=0$, the recursion
$f_m=(1-\lambda_m)f_{m-1}+\lambda_m g_m$ implies $f_m\in\mathrm{span}(S)$ and hence
$r_m=y-f_m\in\mathrm{span}(S)$ for all $m$.

Fix $u\in\mathrm{span}(S)$ and write $u=\sum_{j=1}^s a_j x_{i_j}$ with $a\in\mathbb R^s$.
Let $G$ be the Gram matrix of $S$, $G_{pq}=\langle x_{i_p},x_{i_q}\rangle$.
By Gershgorin's circle theorem (\cite{HJ}),
\[
\lambda_{\min}(G)\ge 1-(s-1)\mu_S>0,
\]
so
\[
\|u\|_2^2=a^\top G a\ge (1-(s-1)\mu_S)\|a\|_2^2.
\]
On the other hand, for a symmetric dictionary the atomic norm satisfies
$\|u\|_{\mathcal A}\le \|a\|_1\le \sqrt{s}\|a\|_2$.
Combining both inequalities yields the norm comparison
\[
\|u\|_2\ge \frac{\sqrt{1-(s-1)\mu_S}}{\sqrt{s}}\,\|u\|_{\mathcal A},
\qquad u\in\mathrm{span}(S).
\]
Applying this to $u=r_m$ gives
\[
\|r_m\|_2\ge \frac{\sqrt{1-(s-1)\mu_S}}{\sqrt{s}}\,\|r_m\|_{\mathcal A}.
\]
Finally, the triangle inequality for $\|\cdot\|_{\mathcal A}$ gives
\[
\|y\|_{\mathcal A}\le \|f_m\|_{\mathcal A}+\|r_m\|_{\mathcal A}
\quad\Longrightarrow\quad
\|r_m\|_{\mathcal A}\ge \|y\|_{\mathcal A}-\|f_m\|_{\mathcal A}.
\]
Combining the last two displays proves the bound.
\end{proof}

\begin{Example}[Orthogonal $s$-sparse target]
Let $S=\{x_1,\dots,x_s\}$ be orthonormal, so $\mu_S=0$, and take
\[
y=\frac1s\sum_{j=1}^s x_j .
\]
Then $\|y\|_{\mathcal A}=1$, and the proposition yields
\[
\|r_m\|_2\ge \frac{1}{\sqrt{s}}\bigl(1-\|f_m\|_{\mathcal A}\bigr).
\]

Thus, any mechanism (such as the product obstruction in
Theorem~\ref{main} for $\alpha>1$) that enforces
$\sup_m\|f_m\|_{\mathcal A}<1$ produces a strictly positive
Euclidean residual floor whose magnitude scales as $1/\sqrt{s}$.
\end{Example}

\section{Relation to other greedy learning methods}

Although the analysis in this work is carried out for the Power--Relaxed
Greedy Algorithm (PRGA), the stagnation mechanism identified in
Theorem~\ref{main} is not specific to this algorithm.
Rather, it arises from a structural property shared by many
stage-wise greedy learning procedures.

Consider iterative methods that construct approximations of the form
\[
f_m = (1-\lambda_m)f_{m-1} + \lambda_m g_m ,
\]
where $g_m$ is selected from a dictionary $\mathcal D$
according to a greedy criterion (for instance maximizing correlation
with the current residual), and $\lambda_m$ is a step-size parameter.
Such update rules appear in a wide range of algorithms in approximation
theory, sparse learning, and machine learning.

A key structural quantity governing the expressive power of these
methods is the cumulative accumulation of updates
\[
A_\infty := \sum_{m=1}^{\infty} \lambda_m .
\]

If $A_\infty=\infty$, the algorithm has an unbounded cumulative
corrective capacity and can, in principle, continue reducing the
residual indefinitely.
However, when $A_\infty<\infty$ (as happens, for example, when
$\lambda_m=m^{-\alpha}$ with $\alpha>1$),
the total mass of corrections applied by the algorithm is finite.
In that case the sequence of iterates remains confined to a bounded
subset of the convex hull of the dictionary atoms, which may lead to
persistent residual bias even in realizable problems.

This mechanism suggests that the phenomenon described in this paper
should be interpreted as a structural limitation of stage-wise greedy
learning under rapidly decaying step sizes.

\paragraph{\textit{Boosting and stage-wise additive models.}}

Boosting algorithms construct predictors as additive combinations of
weak learners, typically updating the model through stage-wise
corrections aligned with the residual.
When the learning rate is chosen too small or decays too rapidly,
the cumulative contribution of new learners may become insufficient
to eliminate residual components, potentially producing persistent bias.

\paragraph{\textit{Frank--Wolfe and projection-free optimization.}}

The Frank--Wolfe algorithm also generates iterates of the form
\[
f_m = (1-\gamma_m)f_{m-1} + \gamma_m s_m ,
\]
where $s_m$ is chosen by linear minimization over a feasible set.
Although classical Frank--Wolfe step-size rules satisfy
$\sum_m \gamma_m=\infty$,
variants employing predetermined or aggressively decaying schedules
may exhibit a similar finite-accumulation phenomenon.

\paragraph{\textit{Greedy sparse approximation.}}

In sparse approximation and matching pursuit–type algorithms,
updates are also performed through greedy selection of atoms followed
by controlled updates.
If the effective contribution of new atoms is progressively diminished
through decaying coefficients, the resulting approximation may remain
trapped away from the target even when the dictionary contains the
exact representation.

From this perspective, the present work should not be viewed as
identifying a pathology specific to PRGA, but rather as revealing a
general structural risk associated with rapidly decaying learning-rate
schedules in stage-wise greedy methods.
Theorem~\ref{main} provides a concrete and analytically tractable
instance of this phenomenon.

\section{Discussion and implications}

We conclude with several remarks concerning the interpretation of
Theorem~\ref{main} and its implications for greedy learning methods.

\paragraph{\textit{Step-size schedules.}}

Theorem~\ref{main} shows that rapidly decaying step-size schedules
with $\alpha>1$ (so that $\sum_{m=1}^{\infty}\lambda_m<\infty$)
may lead to structural stagnation even in realizable problems.
This suggests a minimal structural requirement for stage-wise methods:
\[
\sum_{m=1}^\infty \lambda_m = \infty.
\]

For power schedules $\lambda_m = m^{-\alpha}$ this condition corresponds to
$\alpha \le 1$.
When $\alpha>1$, the total corrective capacity of the algorithm
remains bounded, and a non-vanishing bias may persist.

At the critical value $\alpha=1$, the accumulation is logarithmic,
allowing the algorithm to continue correcting the residual,
although progressively more slowly.
For $0<\alpha<1$, the accumulation grows polynomially,
allowing a more sustained correction over time,
while constant step sizes correspond to linear accumulation.

These observations suggest that, from a structural point of view,
step-size schedules satisfying $\alpha\le1$ are more appropriate when
exact recovery is desired in noiseless sparse learning settings.

\paragraph{\textit{Interaction with stochastic noise.}}

The stagnation phenomenon described in Theorem~\ref{main}
is derived in a noiseless realizable setting.
It is natural to ask how additive stochastic noise would interact
with this behavior.

Suppose that the observed target is
\[
y^\varepsilon = y + \varepsilon,
\]
where $\varepsilon$ denotes a perturbation (deterministic or random)
with $\mathbb E\|\varepsilon\|_2^2<\infty$.

Since the PRGA update is linear in the residual and the cumulative
step-size $\sum_m \lambda_m$ is finite for $\alpha>1$,
the total correction mass available to compensate both the signal
and the noise remains bounded.

Consequently the algorithm cannot completely eliminate either the
structural bias induced by finite accumulation or the noise component.
This suggests that overly aggressive step-size decay may lead to
persistent bias even in stochastic settings.

\section{Numerical experiments}

In this section we provide numerical evidence supporting the theoretical
stagnation result established in Theorem~\ref{main}. 
The goal of the experiments is to verify that, for $\alpha>1$,
the Power--Relaxed Greedy Algorithm (PRGA) fails to converge to zero
training error even in simple realizable sparse regression problems.
All simulations were implemented in Python.

\subsection{Experimental setup}

We consider a synthetic regression problem in $\mathbb R^n$ with $n=200$.
For each coherence level $\mu \in [0,0.95]$, we construct two unit vectors
$x_1, x_2 \in \mathbb R^n$ satisfying
\[
\langle x_1, x_2 \rangle = \mu.
\]

The upper limit $\mu\le 0.95$ simply avoids the nearly colinear regime
$\mu\approx 1$, where the two atoms become almost indistinguishable and
the geometry of the problem becomes numerically degenerate.

The target vector is defined as $$
y = (1-b)x_1 + b x_2,
\qquad b=0.25,$$
which lies exactly in the span of $\{x_1,x_2\}$.
Hence the regression problem is perfectly realizable and admits
a sparse two-atom representation.

We run PRGA over the symmetric dictionary $\mathcal D = \{\pm x_1, \pm x_2\},$
with step sizes $\lambda_m = m^{-\alpha},$ for different values of $\alpha$. Each experiment is run for $M=800$ iterations, which is sufficient
for the residual norm to stabilize and therefore approximates the
asymptotic behavior of the algorithm.

For $\alpha>1$ we also compute the theoretical lower bound predicted by
Theorem~\ref{main},
\[
b(1-\mu)\sqrt{\frac{1+\mu}{2}}\,P_\alpha,
\qquad
P_\alpha = \prod_{k=2}^{\infty}\left(1-\frac{1}{k^\alpha}\right).
\]

The infinite product $P_\alpha$ is approximated numerically using a
logarithmic summation of $\sum_{k=2}^N \log(1-k^{-\alpha})$ followed by
exponentiation, which provides a numerically stable approximation of the
product (see Appendix \ref{app} and \cite{Higham2002}).

\subsection{Stagnation as a function of coherence}
For $\alpha>1$ the residual norm remains strictly bounded away from zero,
in agreement with the theoretical prediction of
Theorem~\ref{main}.
Moreover, the empirical curves closely follow the theoretical lower bound,
confirming the quantitative dependence
\[
\inf_m \|r_m\|_2
\gtrsim
b(1-\mu)\sqrt{\frac{1+\mu}{2}}\,P_\alpha.
\]

As $\mu \to 0$ (nearly orthogonal features),
the stagnation level approaches the value predicted for the orthogonal case.
As $\mu$ increases, the lower bound decreases, but the non-convergence
phenomenon persists for all $\mu<1$.

\begin{figure}[h]
\centering
\includegraphics[width=0.65\linewidth]{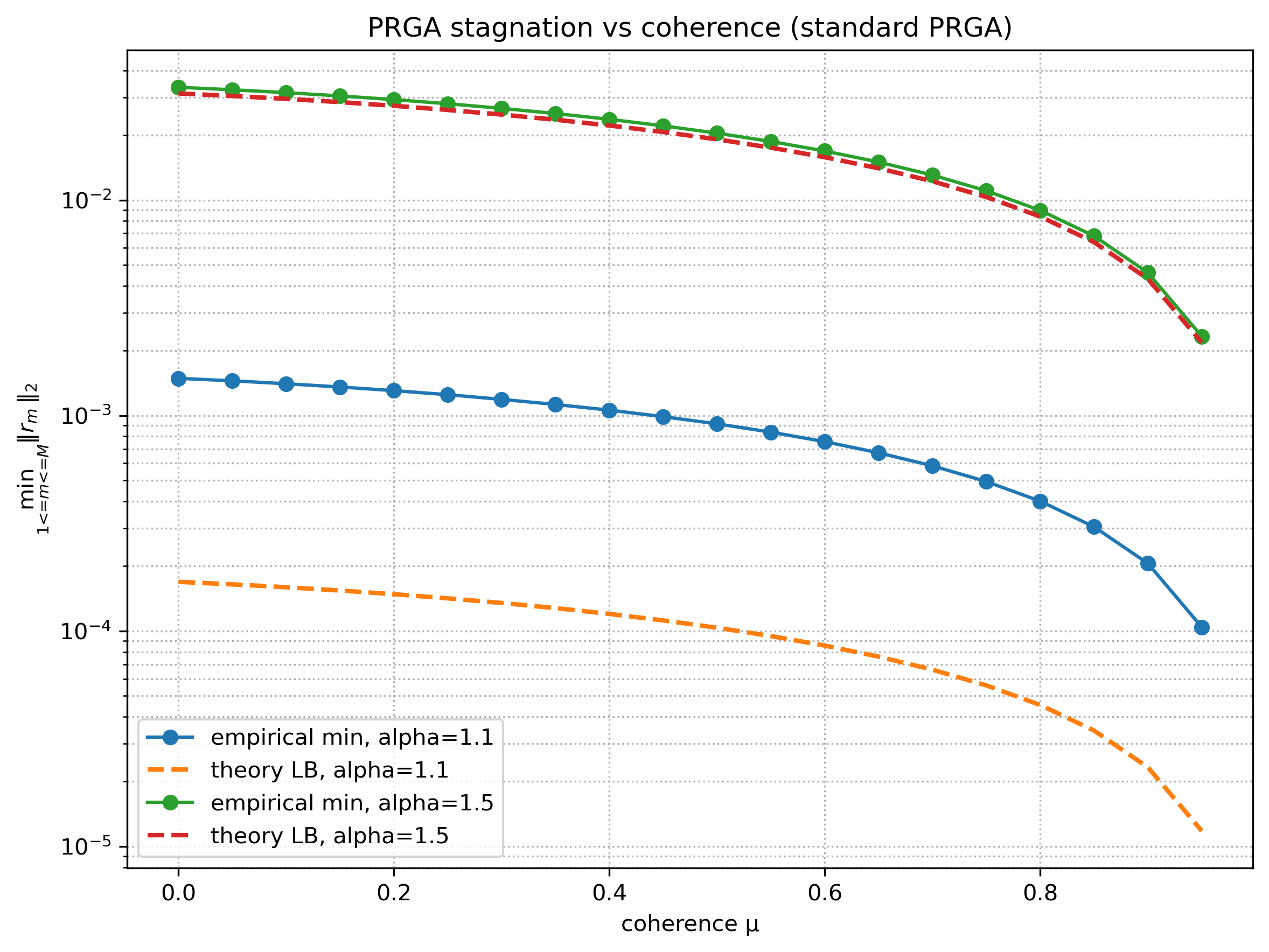}
\caption{
Minimum residual norm $\min_{1\le m \le M}\|r_m\|_2$
as a function of the coherence $\mu$ for $\alpha=1.1$ and $\alpha=1.5$.
Solid lines correspond to the empirical PRGA performance,
while dashed lines indicate the theoretical lower bound
$b(1-\mu)\sqrt{\frac{1+\mu}{2}}\,P_\alpha$.
}
\label{fig:stagnation_mu}
\end{figure}

\subsection{Dependence of stagnation on $\alpha$}

We next examine how the stagnation level varies as a function of the
decay parameter $\alpha$.
To this end we fix the coherence parameter $\mu=0.2$
and evaluate the final residual norm $\|r_M\|_2$
for $\alpha$ varying in the interval $[1.1,2.0]$.

The next figure compares the empirical residual
levels with the theoretical prediction derived from
Theorem~\ref{main}. The results show that the stagnation level increases with $\alpha$,
reflecting the fact that faster step-size decay reduces the cumulative
corrective capacity of the algorithm.
The empirical residual closely tracks the theoretical lower bound,
confirming that the infinite product $P_\alpha$ correctly predicts
the dependence of the stagnation floor on the decay parameter.

\begin{figure}[h]
\centering
\includegraphics[width=0.65\linewidth]{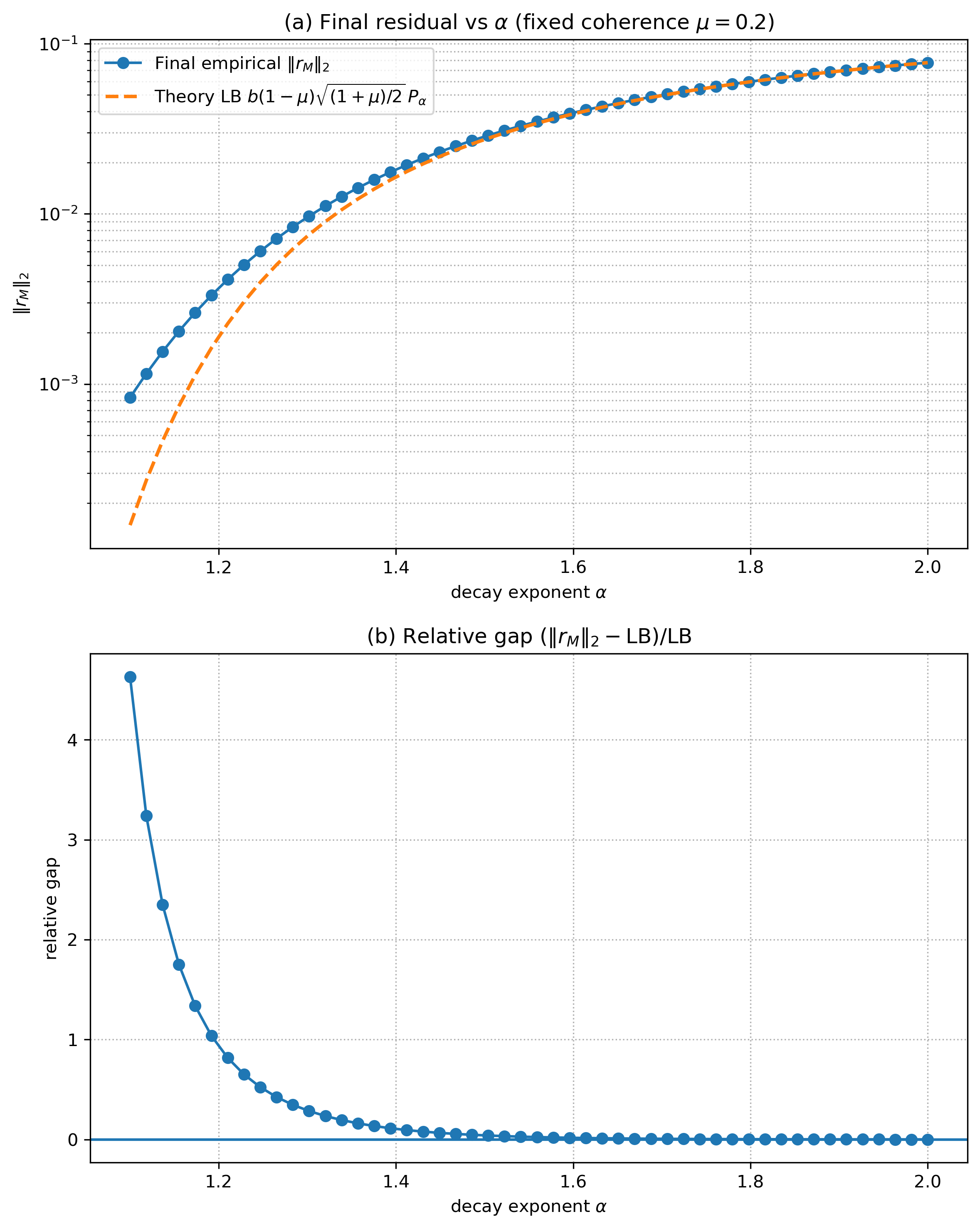}
\label{fig:stagnation_alpha}
\end{figure}
\section{Conclusion}

This work investigates the behavior of the Power--Relaxed Greedy Algorithm
(PRGA) under rapidly decaying step-size schedules.
Although the non-convergence of PRGA for $\alpha>1$ was previously known
in an abstract Hilbert space setting \cite{ref-prga}, our results show that
this phenomenon has a clear and concrete interpretation in sparse learning
problems.

We analyze a realizable regression model involving two atoms of prescribed
coherence and establish an explicit lower bound on the residual norm.
Theorem~\ref{main} shows that when the step sizes decay as
$\lambda_m=m^{-\alpha}$ with $\alpha>1$, the residual cannot converge to
zero and remains bounded away from zero by a quantity proportional to the
product
\[
P_\alpha=\prod_{k=2}^{\infty}\left(1-\frac{1}{k^\alpha}\right).
\]
This result demonstrates that the failure of convergence is not an
artifact of high-dimensional geometry or statistical noise, but rather a
structural consequence of finite cumulative step-size accumulation.

The analysis also provides a geometric interpretation of the phenomenon.
Through the atomic-norm framework, the iterates of PRGA remain confined to
a bounded region determined by the cumulative step sizes.
Combined with incoherence assumptions, this yields Euclidean lower bounds
for sparse residuals, illustrating how stagnation scales with the sparsity
level and the geometry of the dictionary.

Numerical experiments confirm the theoretical predictions.
In particular, the empirical residual closely follows the theoretical
lower bound as a function of both the coherence parameter $\mu$ and the
decay exponent $\alpha$, supporting the role of the product $P_\alpha$
as the key quantity governing the stagnation level.

From a broader perspective, these results highlight a structural trade-off
in stage-wise greedy learning methods.
While relaxation parameters are often introduced to improve stability,
an overly aggressive decay may limit the cumulative corrective capacity of
the algorithm.
Ensuring that the total step-size accumulation remains unbounded therefore
appears to be an important structural condition for reliable convergence
in greedy sparse approximation and related learning procedures.

\appendix
\section{Auxiliary results on the infinite product}\label{app}
To try to understand better Theorem \ref{main}, we present to basic lemmas. The following one could be founded in \cite{ref-garcia} about the positiviy of the inifinite product.
\begin{Lemma}
Let $\alpha>1$ and define
\[
P_\alpha := \prod_{k=2}^{\infty}\left(1-\frac{1}{k^\alpha}\right).
\]
Then $P_\alpha \in (0,1)$.
\end{Lemma}

\begin{proof}
Since $\alpha>1$, the series
\[
\sum_{k=2}^{\infty}\frac{1}{k^\alpha}
\]
converges. 

For $0<x<1$, we have the expansion
\[
\log(1-x) = -x + O(x^2),
\]
hence
\[
\sum_{k=2}^{\infty}
\log\!\left(1-\frac{1}{k^\alpha}\right)
= -\sum_{k=2}^{\infty}\frac{1}{k^\alpha}
+ O\!\left(\sum_{k=2}^{\infty}\frac{1}{k^{2\alpha}}\right).
\]
Both series on the right-hand side converge, so the logarithmic series converges to a finite negative number. Therefore,
\[
P_\alpha
= \exp\!\left(
\sum_{k=2}^{\infty}
\log\!\left(1-\frac{1}{k^\alpha}\right)
\right)
\in (0,1).
\]
\end{proof}

The next one is a lemma talking about the stable logarithmic computation of $P_\alpha$ that we have used in the numerical code of python (see \cite{zenodo}).

\begin{Lemma}
Let $\alpha>1$. For $K\ge2$ define
\[
P_{\alpha,K} := \prod_{k=2}^{K}\left(1-\frac{1}{k^\alpha}\right).
\]
Then
\[
\log P_{\alpha,K}
=
\sum_{k=2}^{K}
\log\!\left(1-\frac{1}{k^\alpha}\right),
\]
and the truncation error satisfies
\[
\left|
\log P_\alpha - \log P_{\alpha,K}
\right|
\le
\sum_{k=K+1}^{\infty}\frac{1}{k^\alpha}.
\]
In particular, the logarithmic summation provides a numerically stable approximation of $P_\alpha$.
\end{Lemma}

\begin{proof}
Taking logarithms converts the product into a sum.
Since for $0<x<1$ we have $|\log(1-x)| \le Cx$ for some universal constant $C$, the tail satisfies
\[
\sum_{k=K+1}^{\infty}
\left|
\log\!\left(1-\frac{1}{k^\alpha}\right)
\right|
\le
C \sum_{k=K+1}^{\infty}\frac{1}{k^\alpha}.
\]
Because $\alpha>1$, the right-hand side converges and decays as $K\to\infty$, proving the claim.
\end{proof}

\textbf{Funding:} The author is supported by the grant PID2022-142202NB-I00 (Agencia Estatal de Investigación, Spain).

\end{document}